\documentclass[%
 reprint,
 amsmath,amssymb,
 aps,
prx,
]{revtex4-2}

\usepackage{xcolor}
\usepackage{amsmath}
\usepackage{amssymb}
\usepackage{multirow}
\usepackage{pifont}
\usepackage{graphicx}
\usepackage{dcolumn}
\usepackage{bm}
\usepackage{bbm}
\usepackage{amsthm}

\newtheorem{theorem}{Theorem}[section]
\newtheorem{corollary}{Corollary}[theorem]
\newtheorem{lemma}[theorem]{Lemma}

\begin{document}

\preprint{APS/123-QED}

\title{Discovering conservation laws from trajectories via machine learning}

\author{Seungwoong Ha}
\author{Hawoong Jeong}%
\altaffiliation[Also at ]{Center for Complex Systems, Korea Advanced Institute of Science and Technology, Daejeon 34141, Korea}
\email{hjeong@kaist.edu}
\affiliation{%
  Department of Physics, Korea Advanced Institute of Science and Technology, Daejeon 34141, Korea
}%

\date{\today}

\begin{abstract}
  Invariants and conservation laws convey critical information about the underlying dynamics of a system, yet it is generally infeasible to find them from large-scale data without any prior knowledge or human insight. We propose ConservNet to achieve this goal, a neural network that spontaneously discovers a conserved quantity from grouped data where the members of each group share invariants, similar to a general experimental setting where trajectories from different trials are observed. As a neural network trained with a novel and intuitive loss function called noise-variance loss, ConservNet learns the hidden invariants in each group of multi-dimensional observables in a data-driven, end-to-end manner. Our model successfully discovers underlying invariants from the simulated systems having invariants as well as a real-world double pendulum trajectory. Since the model is robust to various noises and data conditions compared to baseline, our approach is directly applicable to experimental data for discovering hidden conservation laws and further, general relationships between variables.
\end{abstract}

\maketitle


\paragraph{Introduction}
Modern science greatly depends on the mathematical modeling of given systems and finding the internal structures between observables. One of the most important concepts in system modeling is the \textit{invariants} that underlie the system dynamics, which provide significant information about structural symmetries and low-dimensional embeddings of the system. Invariants and symmetries are fundamental building blocks of nearly all physical systems in nature, such as classical systems with Hamiltonians, Gauge orbits, and many other dynamical systems. Scientists have long attempted to identify the hidden correlations and interactions among the observables of such systems by discovering the conserved quantities and underlying symmetries. 

Recently, with the advent of large-scale data and phenomenal advances in machine learning in physical sciences \cite{carrasquilla2017machine, ch2017machine, van2017learning, zhang2017quantum, carleo2017solving,baldi2014searching, ponte2017kernel, zhang2018machine, sun2018deep, torlai2018neural, rafayelyan2020large, amey2021neural}, various studies have contributed towards the \textit{automation of science} \cite{king2009automation}, referring to current efforts to reveal scientific concepts and construct models solely from observed data without human intervention \cite{bongard2007automated, schmidt2009distilling, kaiser2018discovering, wu2019toward, li2019data, champion2019data, iten2020discovering, decelle2019learning, mototake2019interpretable,wetzel2020discovering, liu2020ai}. Following this line, several studies have attempted to accomplish the automated discovery of conserved quantities with neural networks \cite{decelle2019learning, mototake2019interpretable,wetzel2020discovering,liu2020ai}; limitations of these works though include the requirements for additional non-automated preprocessing and often a great number of datasets from different conditions, as well as the ability to only infer the number of invariants. Real-world empirical data are often sparse, noisy, and scattered into small groups, and hence a model for automated discovery needs to be robust to such harsh conditions.

\begin{figure}
  \centering
  \includegraphics[width=\linewidth]{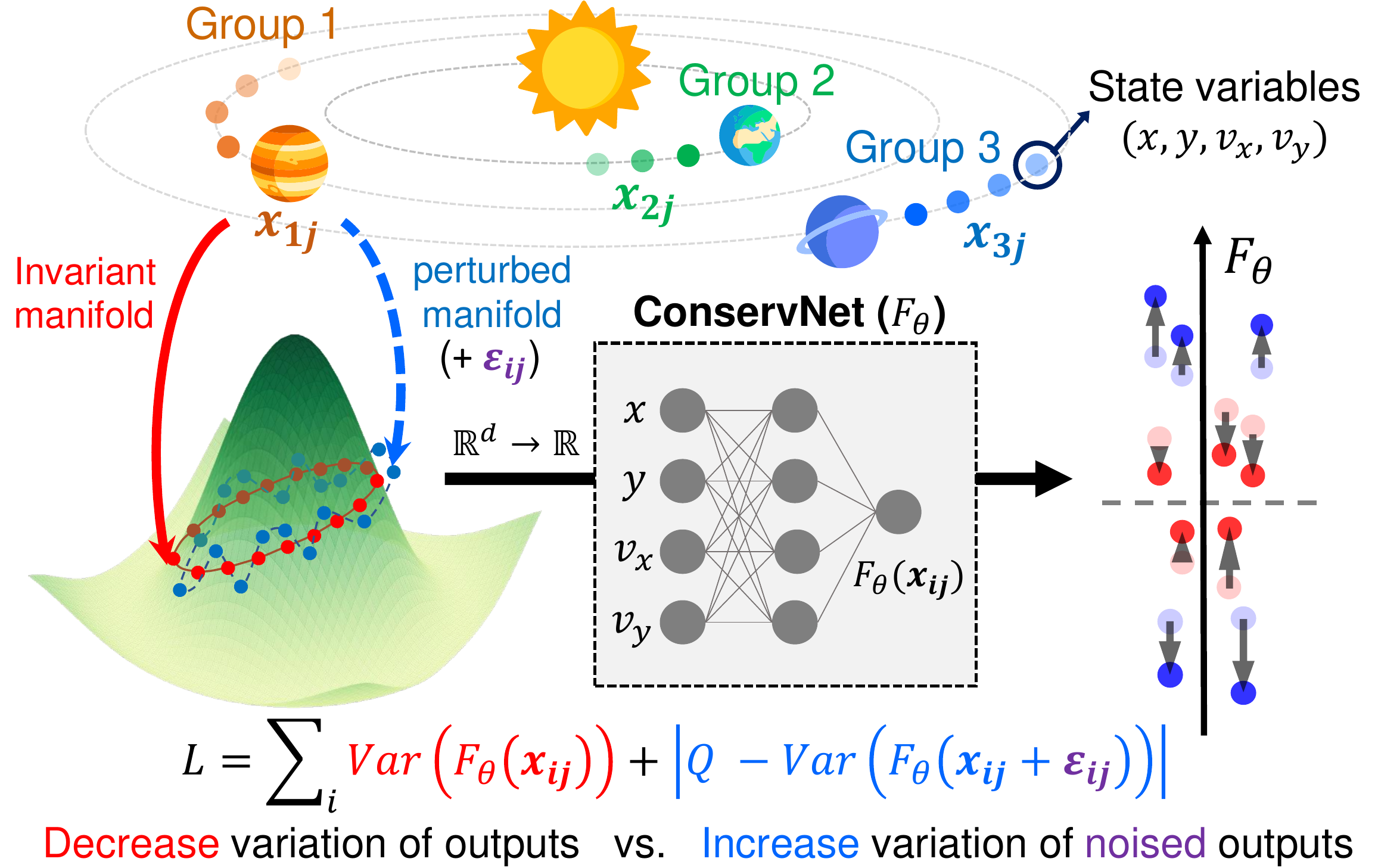} 
  \caption{Schematic overview of ConservNet and the role of noise-variance loss. Each group of data, which is a time series of planet trajectories in this example, is fed into model and optimized to minimize the noise-variance loss.}\label{fig:1}
\end{figure}

In this study, we introduce ConservNet, a neural network to discover conserved quantities in grouped data, such as trajectories, without any prior knowledge of the system. Instead of explicitly restricting the model to ensure certain symmetries, we propose a novel loss function that facilitates the model to directly learn the invariant function. We show that ConservNet robustly finds a invariant by reducing the intra-group variance of its output while preventing convergence into trivial constant functions. Our model can be applied to a variety of realistic data conditions with multiple groups, is robust to noises and nuisance variables, and employs a pipeline from raw data to invariants in an end-to-end manner that enables the direct extraction of symbolic formulas. We examine the capability of ConservNet by applying it to five model systems ranging from synthetic invariants to physical models that cover diverse functional forms, along with experimental trajectory data of a double pendulum. The robustness of our method strongly demonstrates the potential of ConservNet to be applied to real systems where data is sparse and no conservation laws are known.

\paragraph{Noise-Variance Loss}
Throughout this paper, the data condition $(N, M)$ indicates that the data is divided into $N$ groups, in which each group shares the same invariant and has $M$ data points. Our goal is to find conserved quantities hidden in such grouped $d$-dimensional data that are expected to have at least one invariant. We assume that the system has an invariant function $V$ that satisfies $V(\mathbf{x}_{ij}) = C_i$ for all $\mathbf{x}_{ij} \in G_i$, where $G_i$ denotes the $i$-th group and $\mathbf{x}_{ij} \in \mathbbm{R}^d$ is the $j$-th input data of dimension $d$ from group $i$.

In order for the model to properly approximate the invariant, it needs to satisfy two important criteria. First, the desired model should produce a ground-truth invariant $C$, or at least a value strongly correlated with the true invariant. Second, the model output from the same group should be equal in the ideal case, or at least its deviation should be minimized.

To satisfy the second criteria, the loss function $L$ for the neural output $F_{\theta}$ should decrease the intra-group variance of the outputs from each group, and thus the variance term $\mathcal{L}_{i, \text{var}} = (M^{-1} {\textstyle\sum}_j F_\theta(\mathbf{x}_{ij})^2) - (M^{-1} {\textstyle\sum}_j F_\theta(\mathbf{x}_{ij}))^2$ should be minimized. Here, the naive optimization of this loss function will generally fall into trivial minima. As an example, the whole class of simple multivariate function $ f:\mathbbm{R}^{d} \rightarrow C_0$ for any real value $C_0$ becomes one of the global minima of $\mathcal{L}_{i, \text{var}}$ since the output is constant regardless of the input. Convergence to such a trivial solution would violate the first criteria in our case.

Thus, we need to guide $F_{\theta}$ to capture a non-trivial invariant besides constant function. In this study, we inhibit trivial convergence by adopting a \textit{spreading} term that increases the variance of the output from improper input, such as perturbated input with noise. This spreading loss can be expressed as $\mathcal{L}_{i, \text{noise}} = \textrm{Var}(F_{\theta}(\mathbf{x}_{ij})) + |Q - \textrm{Var}(F_{\theta}(\mathbf{x}_{ij}+\bm{\varepsilon}_{ij}))|$, where $Q$ is the spreading constant and $\bm{\varepsilon}_{ij}$ denotes a random noise vector, which its $L_2$ norm is bounded to $R=\text{max}(||\bm{\varepsilon}_{ij}||_{2})$. Here, $Q$ restricts the absolute value of the variance of the outputs from perturbed inputs, since optimization without this constraint will lead $F_\theta$ into a diverging function, ignoring the variance minimization term. Thus, the relative scale of $Q$ and $R$ controls the fineness of the spreading. Similar intuition for spreading loss can be found in a contrastive loss in self-supervised learning \cite{sun2015deep, chen2020simple, khosla2020supervised}, which also needs to increase distance in representation space between different classes while preventing divergence. Combining two terms and summing over all groups, the loss function for ConservNet becomes 

\begin{equation}
\mathcal{L} =\sum_{i} \mathcal{L}_i = \sum_{i} \textrm{Var}(F_{\theta}(\mathbf{x}_{ij})) + |Q - \textrm{Var}(F_{\theta}({x}_{ij}+{\bm{\varepsilon}}_{ij}))|. \label{eq:1}
\end{equation}

We propose this new loss function for capturing an invariant as noise-variance (NV) loss, as schematically depicted in Fig. \ref{fig:1}. We prove that two adversarially competing terms in NV loss inhibits trivial convergence by preventing the gradient $\nabla F_\theta$ from becoming $\vec{\bm{0}} \in \mathbbm{R}^d$ \cite{supplemental}, which its implication can be physically interpreted if the system has a well-defined Hamiltonian $H$. In the language of Hamiltonian mechanics, the model aims to learn a \textit{constant of motion} $G$ with various energy levels, which is a generating function of the (infinitesimal) canonical transformation that leaves given $H$ invariant \cite{goldstein2002classical}. This implies that $\frac{dG}{dt} = \{G, H\} = \frac{\partial G}{\partial \mathbf{q}}\frac{\partial H}{\partial \mathbf{p}} - \frac{\partial G}{\partial \mathbf{p}}\frac{\partial H}{\partial \mathbf{q}} = 0$, where $\mathbf{p}$ and $\mathbf{q}$ are generalized positions and momenta. If $\nabla G = \vec{\bm{0}}$, then $\frac{dG}{dt}$ becomes zero regardless of the form of the Hamiltonian, and such $G$ represents stationary transformation which conveys no information about the system. In this sense, spreading loss thus promotes the model to learn non-trivial canonical transformation by letting the model output from the set of non-canonical transformations, namely, the perturbed trajectory cannot form a constant of motion by a margin of $Q$. 

\renewcommand{\arraystretch}{1.3}
\begin{table}\centering
\caption{Systems and invariants for verification. We use $\alpha, \beta, \delta,\gamma = (1.1, 0.4, 0.1, 0.4)$ for the Lotka--Volterra system and $m=1, GM=1$ for the Kepler problem. For the double pendulum case, the ideal Hamiltonian is given.}
  \begin{tabular}{l|c}
         System  & Invariant formula \\ \hline
         S1      & $C = x_1 - 2x_2x_3 + 3x_4^2$     \\
         S2      & $C = 3x_1 + 2\sin(x_2) + \sqrt{|x_1|}x_3^3$  \\
         S3     & $C = 2x_1x_2 - (\ln(|x_1+x_3|)-x_4)/x_3$ \\ \hline
         Lotka--Volterra &  $C = \alpha\ln(x) + \delta\ln(y) - \beta x - \gamma y $ \\ \hline 
         \multirow{3}{*}{Kepler problem} & $C_1 = xv_y - yv_x$  \\
         & $C_2 = \frac{1}{2}m(v_x^2+v_y^2) - \frac{GMm}{r}$ \\
         & $C_3 = \mathbf{p}\times\mathbf{L}-mk\hat{\mathbf{r}}$\\ \hline
         \multirow{3}{*}{Double pendulum} &  $C_{\text{ideal}} = L_1^2(m_1+m_2)\omega^2 + m_2L_2^2\omega^2$ \\
         & $ + 2m_1m_2L_1L_2\omega_1\omega_2\cos(\theta_1-\theta_2) $\\
         (experiment)& $ - 2gL_1(m_1+m_2)\cos(\theta_1)-2gm_2L_2\cos(\theta_2)$ 
         \label{table:1}
  \end{tabular}
\end{table}

\begin{figure}[t]
\centering
\includegraphics[width=\linewidth]{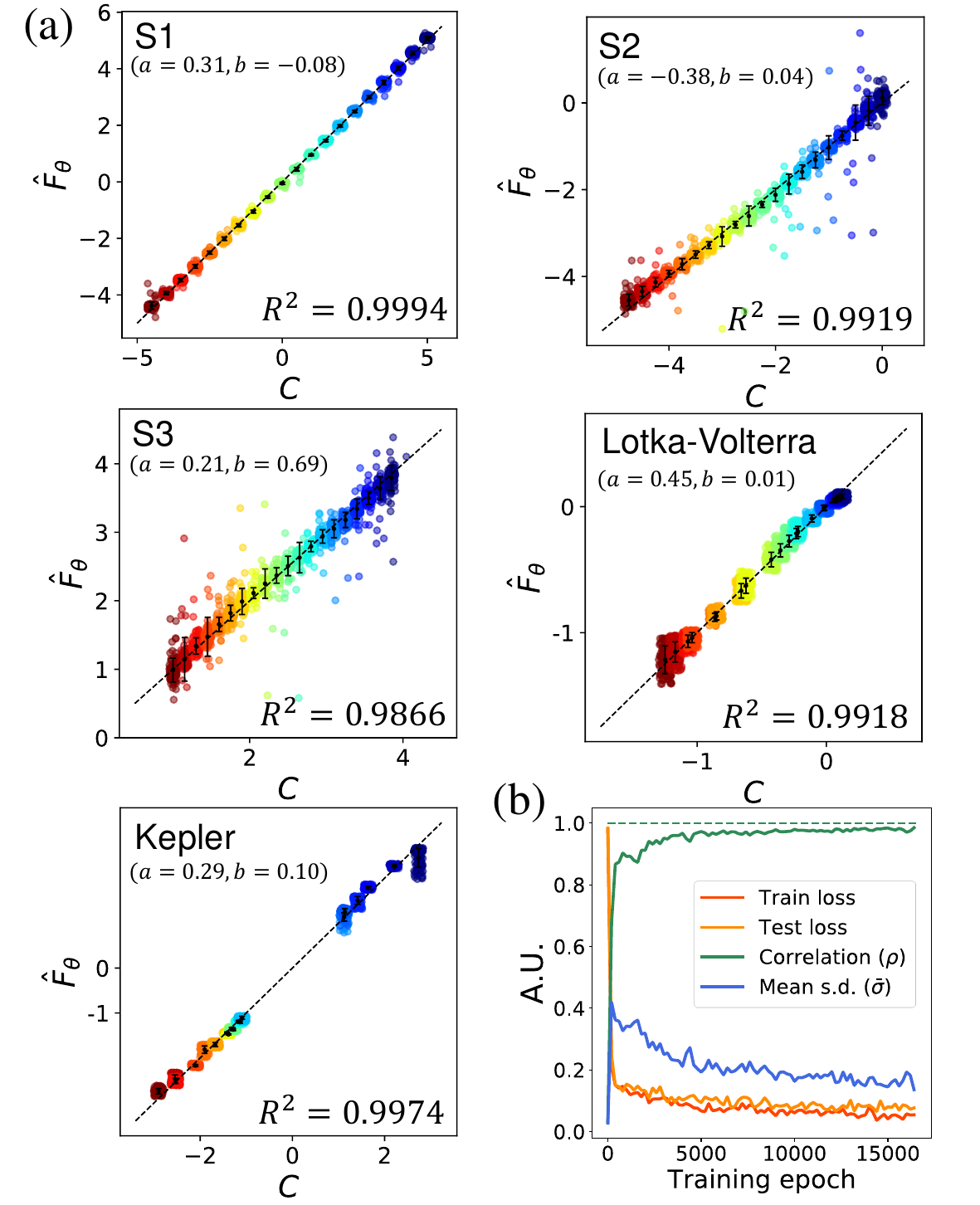}
\caption{Model performances of ConservNet. (a) Ground-truth invariants $C$ versus fitted ConservNet outputs $\hat{F}_\theta = aF_\theta+b$ for $S1$, $S2$, $S3$, the Lotka--Volterra equation, and the Kepler problem are plotted under data condition $(20, 100)$ with $R^2$ values. Points with the same color share the same invariant values but are plotted at jittered values for visualization. The mean output value of each group (black dot) with error bars for standard deviation and an identity line (dotted) drawn for comparison. (b) Result statistics for invariant $S2(20, 100)$ with ideal correlation $1$ (green, dashed).} \label{fig:2}
\end{figure}

\paragraph{Neural model construction and training}
ConservNet is a feed-forward neural network constructed with $4$ hidden layers with a layer width of $320$ neurons and a single output neuron, using Mish \cite{misra2019mish} as an activation function. Our model receives system data $\mathbf{x}_{ij}$ and produces a single scalar value $F_{\theta}(\mathbf{x}_{ij})$ that aims to approximate the mapping function from states to conserved quantities. The noise vector $\bm{\varepsilon}_{ij}$ is newly sampled from the multivariate uniform distribution at every batch with the proper scaling. In practice, we employ standard deviation $\sigma(\mathbf{x}) = \sqrt{\text{Var}(\mathbf{x})}$ instead of variance $\text{Var}(\mathbf{x})$ as a measure of variance.

As a baseline for comparison, we trained a recently proposed Siamese neural network (SNN) \cite{wetzel2020discovering} along with our model. This SNN architecture extracts an invariant by classifying whether two data points are from the same instance or not, similar to \cite{decelle2019learning}. Both ConservNet and the SNN are trained with Adam \cite{kingma2014adam} optimizer using \texttt{PyTorch} \cite{paszke2019pytorch} for $50,000$ epochs with early stoppings. For all experiments, $Q=1$ and spreading noise vector $\bm{\varepsilon}_{ij}$ is sampled from the uniform random vector with the maximum norm $R=1$ \cite{supplemental}.

\paragraph{Model systems and datasets}

In this study, the ability of ConservNet is tested with three synthetic systems, two simulated model systems, and a real double pendulum dataset from \cite{schmidt2009distilling}. The functional form of each invariant is presented in Table \ref{table:1}. Three synthetic systems $S1$, $S2$, and $S3$ are constructed to show a variety of functional forms such as cubic, trigonometric, logarithmic, and rational functions. For the Lotka--Volterra system ($\frac{dx}{dt} = \alpha x - \gamma xy$, $\frac{dy}{dt} = - \beta y + \delta xy$) \cite{takeuchi1996global} and the Kepler problem ($H_{\text{Kepler}} = \frac{\mathbf{p}^2}{2m}-\frac{GMm}{r}$), data are simulated by numerical integration with Euler's method. We find that normalizing the scale between variables improves performances, and thus variables with maximum values exceeding $10$ are rescaled by a factor of $0.1$ \cite{supplemental}.

\begin{figure}[t]
  \centering
  \includegraphics[width=\linewidth]{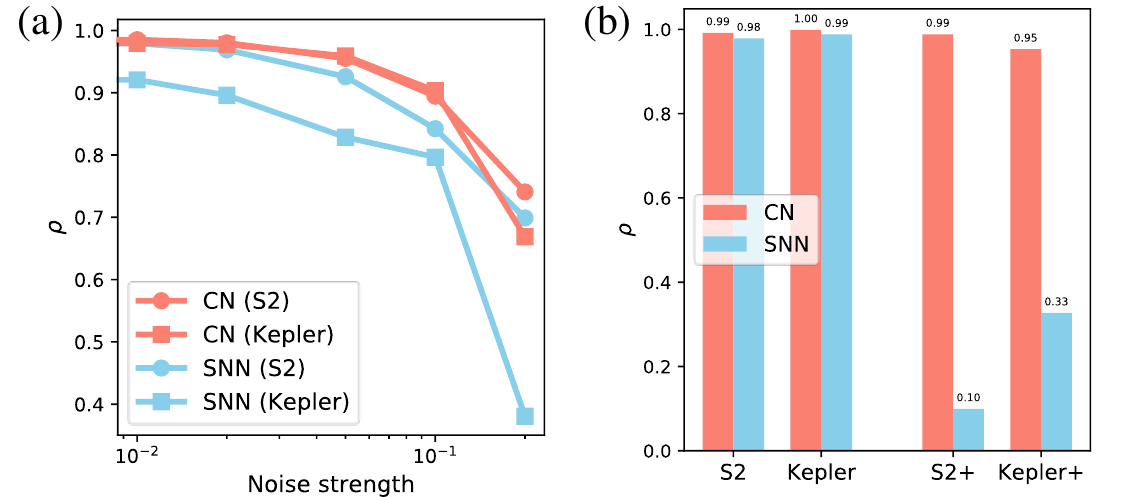}
  \caption{Robustness of ConservNet (CN). (a) Pearson correlation of ConservNet and SNN for invariant $S2$ with various noise strengths. (b) Pearson correlation of ConservNet and SNN for two original datasets ($S2$ and Kepler) and their reinforced versions ($S2+$ and Kepler$+$) that include nuisance variables not appearing in the invariants.}\label{fig:3}
  \end{figure}

  \begin{figure*}[t]
    \centering
    \includegraphics[width=\linewidth]{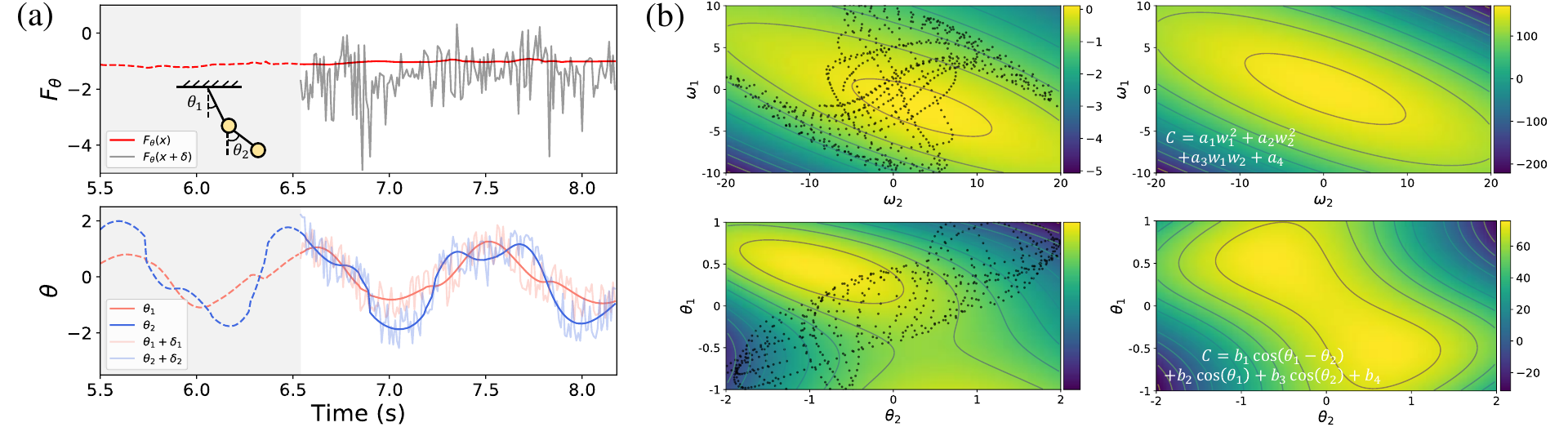}
    \caption{ConservNet results for real double pendulum data. (a) Model output $F_\theta(x)$ and the noised model output $F_\theta(x+\bm{\varepsilon})$ with $\bm{\varepsilon} = (\varepsilon_1, \varepsilon_2)$ (top), and double pendulum trajectories $\theta_1, \theta_2$ and noised trajectories (bottom) versus time. Data in the shaded area are used for training (from $0$ s to $6.54$ s), with the remaining data used for testing ($6.54$ s to $8.18$ s). (b) 2D heatmap of model output $F_\theta$ (left) and ideal Hamiltonian (right) for $(\theta_1, \theta_2, \omega_1, \omega_2) = (0, 0, \omega_1, \omega_2)$ (top) and $(\theta_1, \theta_2, \omega_1, \omega_2) = (\theta_1, \theta_2, 5, 10)$ (bottom). The training data points are scattered in the left panels, while the ideal formulas for the cross-section are presented in the right panels. Here, the ideal heatmaps are drawn with constants $(a_1, a_2, a_3, a_4 = 1, 0.32, 0.82, -170.95)$ and $(b_1, b_2, b_3, b_4 = 41, -124.13, -46,82, 57)$, provided by \cite{schmidt2009distilling}.}\label{fig:4}
    \end{figure*}

\paragraph{Results} 
We prepare $2,000$ training data with various data conditions $(N, M)$ and an equal number of test data for all simulated systems, which is notable as a small amount compared to modern deep learning and other related studies \cite{mototake2019interpretable,wetzel2020discovering,liu2020ai} that typically employ more than $10,000$ data. These conditions are addressed to replicate practical situations with high data costs and a limited number of different observations, common in physical and biological data. The code for dataset generation and model training is publicly available at \cite{code}.

The model performance of ConservNet is evaluated by the aforementioned two criteria: high correlation with the ground-truth invariant and small intra-group variance. We use Pearson correlation $\rho$ and mean intra-group standard deviation $\bar{\sigma} = \frac{1}{N}\sum \sigma_i$ for each criterion. 

Figure \ref{fig:2}(a) illustrates the notable performances of ConservNet, simultaneously finds invariants from multiple groups at once, achieving strong Pearson correlation and small intra-group variation in every model system. For the case of multiple invariants in the Kepler problem, ConservNet captures the angular momentum first and finds the energy secondly when the angular momentum is controlled (see \cite{supplemental} for analysis on multiple invariants). Figure \ref{fig:2}(b) shows result statistics of ConservNet for $S2(20, 100)$ as an example case. We can observe that our model shows smooth convergence without overfitting, while its $\bar{\sigma}$ decreases and $\rho$ approaches to $1$. ConservNet shows consistent performance for different data conditions $(N, M)$ as presented in \cite{supplemental}.

We further investigate the capability and robustness of ConservNet by applying several different conditions prevalent in experimental data. First, we check the impact of noise on the datasets by adding noise $\mathcal{N}(0, s)$ with various strengths $s$. Figure \ref{fig:3}(a) shows that ConservNet gives consistent performances under the noised condition, with better correlation compared to the baseline. ConservNet effectively increases its data size by adding new random noise to its dataset for each batch, thus shows better sample efficiency \cite{xie2020self} and performance with inherent robustness to noise \cite{ lopes2019improving}. We also find that if the data has no invariant, our model alerts it by strong overfitting \cite{supplemental}.

In a real scenario, there might be irrelevant variables in an observed dataset that do not compose the invariant. Filtering out such nuisance variables is crucial for data-driven discovery without any prior knowledge. We test our model with two reinforced datasets. First, we concatenate one extra variable $x_4 \sim \mathcal{N}(0, 1)$ to the $S2$ dataset to construct $S2+$ with a noisy variable. Second, we transform the data of the Kepler problem from Cartesian coordinates $(x, \dot{x}, y, \dot{y})$ into polar coordinates $(r, \dot{r}, \theta, \dot{\theta})$ to construct Kepler$+$. In the polar coordinates, $\theta$ becomes a cyclic coordinate and neither $\dot{r}$ nor $\theta$ appears in angular momentum $r\dot{\theta}$, different from the original Cartesian form $x\dot{y} - y\dot{x}$ where all of the state variables appear. As Fig. \ref{fig:3}(b) shows, ConservNet exhibits robust performances even with the existence of the nuisance variables and coordinate transformation, while the SNN strongly overfits and shows low performance when there are unused variables, possibly due to the nature of classifiers and the absence of a proper regularizer.

Finally, we apply our model to a real double pendulum trajectory from \cite{schmidt2009distilling}, which is a challenging task in a number of ways. According to \cite{schmidt2009distilling}, the data does not \textit{strictly} obey any conservation laws due to noise and friction. Furthermore, the model has to discover the invariant in an extreme data condition where only a single trajectory ($N=1$) with a limited number of data points ($M=654$) is available for training. Note that the SNN is \textit{inapplicable} to this case since it needs at least two groups of data to compare ($N \geq 2$).

We train our model and examine its output for stability and accuracy. Figure \ref{fig:4}a shows that ConservNet output $F_\theta$ stably remains constant for the training and test trajectory but not for the noised trajectory, verifying that ConservNet falls into neither trivial convergence nor overfitting and properly captures the functional form of the invariant. We further check two-dimensional cross-sections of the model output by fixing two variables and varying two variables, and compare them with the cross-sections of the ideal four-dimensional Hamiltonian, constructed with the constants from \cite{schmidt2009distilling}. The results are shown in Fig. \ref{fig:4}b. Considering inherent frictions and the restricted regions of the data points, both heatmaps are similar enough to the point where the inference of the abstract functional form is possible. To summarize, ConservNet successfully captured the conserved quantity from a real double pendulum system with extreme data conditions.

\paragraph{Conclusion and Outlook}

In a real practice where the ground-truth invariant is unknown, we may identify the symbolic form of the invariant by sorting the output values and employing off-the-shelf polynomial regression or symbolic regression algorithms. We illustrate a result of such application for invariant $S1$ as an example in \cite{supplemental}, in which the ground-truth symbolic formula was successfully retrieved.

One limitation that ConservNet shares with \cite{wetzel2020discovering} is that the single model finds a single invariant even if the system could have multiple invariants. While we showed that training with modified data leads to the discovery of remaining invariants, such modification is usually difficult for experimental data. Since our model identifies the numerical value of the invariant and \cite{liu2020ai} approximates the number of invariants, unifying the advantages of these approaches would be an interesting future direction to be explored.

In this letter, the invariants in a classical sense such as a well-defined Hamiltonian is mainly discussed. We can further expand the scope of ConservNet, for instance, by converting any non-autonomous system of $x$ with multiple exogenous variables $y_1, y_2, \dotsc, y_n$ to a standard form of $f(\frac{dx}{dt}, x, y_1, y_2, \dotsc, y_n) = 0$, a system with invariant of value zero \cite{kloeden2011nonautonomous}. In this perspective, one can identify interactions among variables by discovering invariants with the proposed model. This opens a wide variety of potential applications of the model in academic disciplines where the underlying dynamics are yet to be discovered, including advanced domains of quantum mechanics \cite{hioe1981n, aharonov2021conservation}, high-energy physics \cite{christ1980conservation}, astronomical science \cite{abbott2016observation} and particle physics \cite{aad2012observation}, which the scale of dataset is exceedingly large that finding any meaningful structure is humanly intractable.

Automation of science with deep learning is a recently emerging field of study with plenty of uncharted research areas. The present work builds an interpretable bridge across the data and scientists by extracting significant information from entangled high-dimensional data as a form of numerical value and symbolic equation, which can be further explained by a physicist. We envision that at some point, a neural network as ConservNet or an integrated framework of such networks would automatically discover truly unseen knowledge from large-scale datasets.

\begin{acknowledgements}
  This research was supported by the Basic Science Research Program through the National Research Foundation of Korea NRF-2017R1A2B3006930.
\end{acknowledgements}

\appendix
\section{Dataset construction}
We prepare total six systems for training: $S1$, $S2$, $S3$, the Lotka--Volterra system, the Kepler problem, and observation data from a real double pendulum. Table S1 in the SI appendix shows the exact range and sample distribution of each variable in every model system. 

\subsection{Synthetic systems}
The datasets are composed by first randomly drawing the relative variables except for the final one, and calculating the last relative variable which preserves the overall conserved quantity. We tried to maximize the variety of simulated data by setting the noise distribution and variable range for each system as differently as possible. While producing each dataset, we restricted the absolute value of the output of the final variable, consequently rejecting some perilous set of variables that forces the last variable into an extremely diverging value.

\subsection{Physical systems}

We generate the data from physical systems by integrating respective differential equations with Euler's method and performing subsampling to the trajectories. For the Lotka--Volterra system, we simulate the dynamics for $100M$ steps for the dataset of batch size $M$ with time interval $dt=0.01$. The obtained data are further subsampled at every $100$ steps, effectively setting the time interval between data points to $1$. We scale $x, y$ in the Lotka--Volterra equation and the position coordinates $x, y$ in the Kepler problem by a factor of $0.1$.

\subsection{Real double pendulum}
Double pendulum data is adopted from \cite{schmidt2009distilling}, where two trials of double pendulum data are available. We use the first trial, consisting of $818$ data points with four-dimensional time-series $(\theta_1, \theta_2, \omega_1, \omega_2)$. Each data point corresponds to $0.01$ s, making the total data length $8.18$ s. The training set consists of the first $654 = 818 \times 0.8$ points, and the test set consists of the remaining $154 = 818 \times 0.2$ points. We scale $\omega_1, \omega_2$ by a factor of $0.1$ to match the $\theta_1$ and $\theta_2$ scale.

\section{Model training}
Both the SNN and ConservNet are composed of six layers of multi-layer perceptrons (MLPs) with a layer width of $320$, where the input dimension varies by the target system and has a single output neuron. Note that the original SNN \cite{iten2020discovering} used two layers of MLPs with a layer width of $160$; we found that increasing the layer depth and width generally increased the overall performances for both SNN and ConservNet. 

During training, we found the training result of SNN significantly varies by random initialization and is highly prone to overfitting. We report that the SNN shows a good performance (training and test accuracy of $100\%$ and $99.75\%$ with correlation $0.998$) in one trial but converges to a meaningless output with strong overfitting (training and test accuracy of $100\%$ and $50.03\%$ with correlation $0.154$) in the very next trial with the same conditions. In several trials, a larger layer width ($320$) exhibits strong overfitting while a relatively shallow one ($160$) shows better generalization performances.

For a fair comparison, we test every combination of learning rates $[0.005, 0.0005, 0.00005]$ and layer widths $[160, 320]$, and report the best performing one (in terms of test accuracy) among five trials for each condition as representative results of the SNN. In the case of ConservNet, we fix the learning rate at $0.00005$ and layer width at $320$ neurons since the performance was robust against both layer width and learning rate. We train both the SNN and ConservNet for $50,000$ epochs with early stopping, Adam\cite{kingma2014adam} optimizer, and no particular regularizer. The batch size for mini-batch training is fixed to $64$ for the SNN and tentative for Conservnet, where its batch size is fixed as respective group size $M$. Training takes several minutes to several hours on a single GeForce GTX 1080, depending on the batch size and early stopping condition.

\section{Proof for Noise-Variance Loss}

\subsection{Simple loss and its limitation}
In this study, we assume that the dataset consist of $N$ groups with group size $M$, and have a meaningful conserved quantity $V$ that satisfies $V(x_{ij}) = C_i$ for all $x_{ij} \in G_i$ where $G_i$ denotes $i$th group.

One may construct a simple loss function with a variance-decreasing term only, such as 

\begin{equation}
  \mathcal{L}_{\text{simple}} =\sum_{i} \mathcal{L}_i = \sum_{i} \textrm{Var}(F_{\theta}({x}_{ij})) \label{eq:1}
\end{equation}

where $x_{ij}$ is an input data. It is simple to show that by performing gradient descent on $\mathcal{L}_{\text{simple}}$ for the network parameters $\theta$, the output from the same group will approach the same constant value. 

\begin{theorem}
  The global minimum of the functional $\mathcal{L}_{\text{simple}}$ is $F_\theta(x_{ij}) = C_i$ with some constant $C_i$ for all group. \label{thm:1}    
\end{theorem}

\begin{proof}[Proof of Theorem \ref{thm:1}]
  Neural model $F_\theta$ tunes the output by optimizing $F_{\theta, ij} = F_\theta(x_{ij})$ through the network parameter $\theta$. The condition of $\mathcal{L}_{\text{simple},i}$ for the stationary point becomes
  \begin{equation}
    \frac{\partial \mathcal{L}_{\text{simple},i}}{\partial F_{\theta, ij}} =\frac{\partial}{\partial F_{\theta, ij}} \textrm{Var}(F_{\theta, ij})) = 0\label{eq:2}
  \end{equation}
  
  for all $F_{\theta, ij}$. By expanding the variance function with $\mu_{i}^F = \frac{1}{M}\sum_j{F_{\theta, ij}}$, we get

  \begin{align}
    \frac{\partial \mathcal{L}_{\text{simple},i}}{\partial F_{\theta, ij}} & = \frac{1}{M} \frac{\partial}{\partial F_{\theta, ij}} \sum_j (F_{\theta, ij} - \mu_{i}^F)^2 \\ \nonumber 
    & = \frac{1}{M} \sum_j 2(F_{\theta, ij}-\mu_i^F)\frac{\partial F_{\theta, ij} - \mu_i^F}{\partial F_{\theta, ij}} \\ \nonumber
    & = \frac{2}{M}(F_{\theta, ij} - \mu_i^F) - \frac{2}{M^2}\sum_j (F_{\theta, ij}-\mu_i^F) \\ \nonumber
    & = \frac{2}{M}(F_{\theta, ij} - \mu_i^F) = 0 \label{eq:3}
  \end{align}

  since $\sum_j(F_{\theta, ij}-\mu_i^F) = 0$. This means $F_{\theta, ij} = \mu_i$ at the only stationary point, indicates that every $F_{\theta, ij}$ is a constant $C_i = \mu_i^F$.
  Also, by checking its second derivative, we get

  \begin{align}
    \frac{\partial^2 \mathcal{L}_{\text{simple},i}}{\partial F_{\theta, ij}^2} & = \frac{\partial} {\partial F_{\theta, ij}} \frac{2}{M}(F_{\theta, ij} - \mu_i^F) \\ \nonumber 
    & = \frac{2}{M}\left(1-\frac{1}{M}\right) > 0 \label{eq:4}
  \end{align}

  since $M$ is a natural number. Hence, this stationary point is a global minimum.
  \begin{corollary}
    The constant function $F_{\text{const}}(x) = C_0$ for any $x \in \mathbbm{R}^d$ satisfies the condition for the global minimum of the functional $L_{\text{simple}}$.
  \end{corollary}

\end{proof}

Obviously, the intra-group variance will be $0$ if all of the model output from the same group becomes the same constant. But the zero intra-group variance is not a sufficient condition for a meaningful invariant, as mentioned in the main manuscript. Any modern deep learning architecture with perceptrons and feed-forward network (includes ConservNet) can express the constant function by reducing the weight to zero and thus ignoring input completely, and hence prone to learn such simple function rather than meaningful invariant. In Fig. \ref{fig:5}, we can see that the model trained by $\mathcal{L}_{\text{simple}}$ falls into this trap; does not properly capture the given invariant and instead shows nearly constant behavior, even though the train and test loss rapidly converged to zero in the early stage of training. 

\subsection{Noise variance loss}
Now, we focus on the proposed Noise-Variance Loss (NV loss) for ConservNet. 

\begin{equation}
  \mathcal{L} =\sum_{i} \mathcal{L}_i = \sum_{i} \underbrace{\textrm{Var}(F_{\theta}({x}_{ij}))}_{\mathcal{A}_i} + |\underbrace{Q - \textrm{Var}(F_{\theta}({x}_{ij}+{\varepsilon}_{ij}))}_{\mathcal{B}_i}|, \label{eq1}
\end{equation}

where $Q$ is a spreading constant and ${\varepsilon}_{ij}$ denotes a spreading noise vector. The function consists of the same term as $\mathcal{L}_{\text{simple}}$ ($\mathcal{A}_i$) and the additional term that keeps the variance of noised output into a certain value ($\mathcal{B}_i$). We want to show that the minimization of this loss function will avoid trivial convergence by flipping its behavior when noised output variance became too low.

\begin{theorem}
  The constant function $F_{\text{const}}(x) = C_0$ for all $x \in \mathbbm{R}^d$ and some constant $C_0$ is not a minima of $\mathcal{L}_i$. \label{thm:2}
\end{theorem}

\begin{proof}[Proof of Theorem \ref{thm:2}]

To find a stationary point, we again apply partial derivative to $\mathcal{L}$. Due to the absolute value in $\mathcal{B}_i$, the partial derivative becomes

\begin{widetext}
\begin{equation}
  \frac{\partial}{\partial F_{\theta, ij}} |\mathcal{B}_i| = \frac{\partial}{\partial F_{\theta, ij}} \sqrt{\mathcal{B}_i^2} = \text{sgn}(\mathcal{B}_i) \frac{\partial \mathcal{B}_i}{\partial F_{\theta, ij}} = -\text{sgn}(\mathcal{B}_i)\left[\frac{\partial}{\partial F_{\theta, ij}}\textrm{Var}(F_{\theta}({x}_{ij}+{\varepsilon}_{ij})) \right]. \label{eq:5}
\end{equation}
\end{widetext}

By expanding above with Taylor expansion, we get

\begin{widetext}
\begin{align}
  -\text{sgn}(\mathcal{B}_i)\left[\frac{\partial}{\partial F_{\theta, ij}}\textrm{Var}(F_{\theta}({x}_{ij}+{\varepsilon}_{ij})) \right] &= -\text{sgn}(\mathcal{B}_i)\left[\frac{\partial}{\partial F_{\theta, ij}}\textrm{Var}(F_{\theta, ij} + \nabla F_{\theta, ij}\varepsilon_{ij} + \mathcal{O}(\varepsilon_{ij}^2)) \right] \\ \nonumber
  & = -\text{sgn}(\mathcal{B}_i)\left[\frac{1}{M} \frac{\partial}{\partial F_{\theta, ij}} \sum_j (F_{\theta, ij} + \nabla F_{\theta, ij}\varepsilon_{ij} - \mu_i^{F\nabla})^2 \right] \\ \nonumber
  & = -\text{sgn}(\mathcal{B}_i)\left[\frac{1}{M} \sum_j 2(F_{\theta, ij} + \nabla F_{\theta, ij}\varepsilon_{ij} - \mu_i^{F\nabla}) \frac{\partial (F_{\theta, ij} + \nabla F_{\theta, ij}\varepsilon_{ij} - \mu_i^{F\nabla})}{\partial F_{\theta, ij}} \right] \\ \nonumber
  & = -\text{sgn}(\mathcal{B}_i)\left[\frac{2}{M}(F_{\theta, ij} + \nabla F_{\theta, ij}\varepsilon_{ij} - \mu_i^{F\nabla}) - \frac{2}{M^2} \sum_j F_{\theta, ij} + \nabla F_{\theta, ij}\varepsilon_{ij} - \mu_i^{F\nabla} \right] \\ \nonumber
  & = -\text{sgn}(\mathcal{B}_i)\frac{2}{M}(F_{\theta, ij} + \nabla F_{\theta, ij}\varepsilon_{ij} - \mu_i^{F\nabla}) \label{eq:6}
\end{align}
\end{widetext}

where $\frac{\partial \nabla F_{\theta, ij}\varepsilon_{ij}}{\partial F_{\theta, ij}} = 0$, and $\mu_i^{F\nabla} = \frac{1}{M}\sum_j F_{\theta, ij} + \nabla F_{\theta, ij}\varepsilon_{ij}$.

Combining two terms, the entire partial derivative has two cases depend on the sign of $\mathcal{B}_i$.

\begin{widetext}
\begin{align}
  \frac{\partial \mathcal{L}}{\partial F_{\theta, ij}} & = \begin{cases}
    \frac{2}{M} \left[(F_{\theta, ij} - \mu_i^F) - (F_{\theta, ij} + \nabla F_{\theta, ij}\varepsilon_{ij} - \mu_i^{F\nabla}) \right] & \text{if} \quad \text{sgn}(\mathcal{B}_i) = 1 \\
    \frac{2}{M} \left[(F_{\theta, ij} - \mu_i^F) + (F_{\theta, ij} + \nabla F_{\theta, ij}\varepsilon_{ij} - \mu_i^{F\nabla}) \right] & \text{if} \quad \text{sgn}(\mathcal{B}_i) = -1. \end{cases} \\ \nonumber
    & = \begin{cases}
      \underbrace{\frac{2}{M} \left[(\mu_i^{\nabla} - \nabla F_{\theta, ij}\varepsilon_{ij}) \right]}_{\mathcal{C}_i} & \text{if} \quad \text{sgn}(\mathcal{B}_i) = 1 \\
      \underbrace{\frac{2}{M} \left[(2F_{\theta, ij} + \nabla F_{\theta, ij}\varepsilon_{ij} - (\mu_i^{F\nabla} + \mu_i^{F}) \right]}_{\mathcal{D}_i} & \text{if} \quad \text{sgn}(\mathcal{B}_i) = -1. \end{cases}
\end{align}
\end{widetext}

where $\mu_i^{\nabla} = \frac{1}{M}\sum_j \nabla F_{\theta, ij}\varepsilon_{ij}$. When the noised output variance is greater than $Q$ ($\mathcal{C}_i$), both term cooperates to reduce the variance of $F_\theta$, regardless of the input. But when the noised output variance became smaller than $Q$, the terms that contributed cooperation cancel out, and the functional now has a global maximum instead of a global minimum. In this regime, the constant function becomes the only solution as follows. 

\begin{lemma} A constant function $F_{\text{const}}(x) = {C}_0$ for all $x \in \mathbbm{R}^d$ and some constant ${C}_0$ is a global maximum of $\mathcal{L}_i$ when $\text{sgn}(\mathcal{B})=1$. \label{lmm:1}
\end{lemma}
\begin{proof}[Proof of Theorem \ref{lmm:1}]
  First, $\mathcal{C}_i$ has its critical point when every $\nabla F_{\theta, ij}\varepsilon_{ij}$ becomes a constant $\mu_i^{\nabla}$. This is a global maximum since any deviation from the constant will decrease $(\mu_i^{\nabla} - \nabla F_{\theta, ij}\varepsilon_{ij})$ as second derivative test shows.

  \begin{align}
    \frac{\partial^2 \mathcal{L}_{i, \text{sgn}(\mathcal{B}_i)=1}}{\partial \nabla F_{\theta, ij}\varepsilon_{ij}2} & = \frac{\partial} {\partial \nabla F_{\theta, ij}\varepsilon_{ij}} \frac{2}{M}(\mu_i^\nabla - \nabla F_{\theta, ij}\varepsilon_{ij}) \\ \nonumber 
    & = \frac{2}{M}\left(\frac{1}{M}-1\right) < 0 \label{eq:4}
  \end{align}

  Since the component of noise vector $\varepsilon_{ij}$ can have an arbitrary value, the only way to satisfy the condition for the global maximum is that $\nabla F_{\theta, ij} = \vec{\bm{0}}$, which means that the function is a constant at everywhere. 
\end{proof}

Now, suppose that the constant function is one of a minimum of $\mathcal{L}_i$. Then, it can only exist at the region where $\text{sgn}(\mathcal{B}_i)=-1$ to Lemma \ref{lmm:1}. But, the constant function always yields $\text{sgn}(\mathcal{B}_i)=1$ since $\mathcal{B}_i = Q - \textrm{Var}(F_{\theta}({x}_{ij}+{\varepsilon}_{ij})) = Q - 0 = Q > 0$. This is a contradictory, and hence the constant function can't be a minimum of the noise-variance loss.
\end{proof}

Intuitively, the noise-variance loss prevents trivial convergence by keeping gradient of $F_{\theta}$ to have a non-zero value, which cannot be accomplished by the constant function. 

In the main manuscript, we described the physical implication of this proof with a view of Hamiltonian mechanics. It should be noted that the analysis in the main manuscript does not restrict our model's possible application to the Hamiltonian system; The very idea of prohibiting zero gradients of the model output with spreading loss is valid for virtually any (non-trivial) invariant function and can be generalized to the system where explicit Hamiltonian is yet to known or undefined, as our results for synthetic systems show.

\begin{figure}
  \centering
  \includegraphics[width=\linewidth]{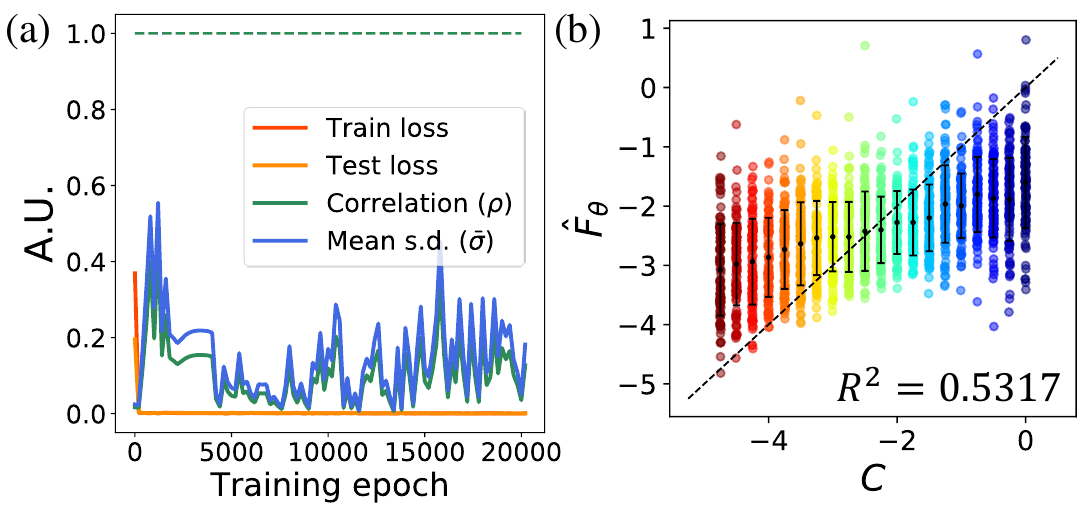}
  \caption{Model performances of ConservNet where the model is trained by $\mathcal{L}_{\text{simple}}$. (a) Train loss, test loss, correlation with ground-truth invariant ($\rho$), and mean intra-group standard deviation ($\bar{\sigma}$) while training $20,000$ epochs of ConservNet for invariant $S2(20, 100)$. Green dashed line indicates ideal correlation value $1$.(b) Ground-truth invariants $C$ versus fitted ConservNet outputs $\hat{F}_\theta = aF_\theta+b$ for the invariant $S2$ is plotted under data condition $(20, 100)$ with the $R^2$ value. Here, $(a,b)=(8.88 \times 10^{-5}, 0.92 \times 10^{-4}$).}\label{fig:5}
\end{figure}

\section{Dataset Detail}

\renewcommand{\arraystretch}{1.3}
\begin{table*}\centering
  \caption{Variable range and sample distribution of train and test data for each system. Variables with a star($*$) at the distribution column are calculated after other variables was drawn from the sample distribution, while dash($-$) indicates a non-samplable variable. For physical systems (Lotka-Volterra and Kepler problem), the distribution column indicates each variable's initial distribution (and hence differs from the actual range as shown). In remarks, the term \textit{rescaled} means the variables are normalized by a given factor before constructing the dataset for the model training.}
  \begin{ruledtabular}
    \begin{tabular}{lcccccc}
      Invariant & System formula & Variable & Distribution & Actual range & Remarks \\ \hline
      \multirow{5}{*}{S1} & \multirow{5}{*}{$C = x_1 - 3x_2x_3 +\frac{1}{2}x_4^2$}&$x_1$ &$*$& $[-4.99, 4.99]$ & \multirow{5}{*}{Model invariant} \\
      & &$x_2$& $\sim \mathcal{N}(0, 2)$ &$[-6.73, 6.90]$  & \\
      & &$x_3$&$\sim \mathcal{N}(0, 2)$& $[-6.39, 5.60]$ & \\
      & &$x_{4} $ & $\sim \mathcal{N}(0, 2)$ & $[-5.41, 5.70]$ & \\
      & &$C$ & - & $[-4.5, 5.0]$  & \\ \hline
      \multirow{4}{*}{S2} & \multirow{4}{*}{$C = 3x_1 + 2\sin(x_2) + \sqrt{|x_1|}x_3^3$ }&$x_1$ &$\sim \mathcal{U}(-3, 3)$& $[-2.99, 0.52]$ & \multirow{4}{*}{Model invariant}\\
      & &$x_2$& $*$ &$[-9.42, 9.42]$ & \\
      & &$x_3$& $\sim \mathcal{U}(-3, 3)$& $[-2.89, 2.90]$ & \\
      & &$C$ & - & $[-5.0, 0.0]$ & \\ \hline
      \multirow{5}{*}{S3} & \multirow{5}{*}{$C = 2x_1x_2 - (\ln(|x_1+x_3|)-x_4)/x_3$ }&$x_1$ &$*$& $[-9.94, 9.99]$ & \multirow{5}{*}{Model invariant}\\
      & &$x_2$ &$\sim \mathcal{U}(-10, 10)$& $[-9.93, 9.98]$ & \\
      & &$x_3$& $\sim \mathcal{U}(0.5, 5)$& $[0.50, 4.99]$ & \\
      & &$x_4$& $\sim \mathcal{U}(-10, 10)$& $[-9.92, 9.98]$ & \\
      & &$C$ & - & $[1.0, 3.85]$ \\ \hline
      \multirow{3}{*}{Lotka-Volterra} & &$x_1$ &$\sim \mathcal{U}(1, 10)$& $[0.004, 25.53]$ & \multirow{3}{*}{$x_1, x_2$ rescaled ($0.1$)}\\
      & $C = \alpha\ln(x) + \delta\ln(y) - \beta x - \gamma y $ &$x_2$ &$ \sim \mathcal{U}(1, 10) $& $[0.03, 9.77]$ & \\
      &  &$C$ & - & $[-1.24, 0.12]$ \\ \hline
      \multirow{5}{*}{Kepler problem} & &$x$ &$*$& $[-5, 5]$ & \\
      & &$y$ &$\sim \mathcal{U}(-5, 5)$& $[-9.74, 15.28]$ & Eccentricity $e<0.99$\\
      &  $C_1 = xv_y - yv_x$ &$v_x$ &$\sim \mathcal{U}(-5, 5)$& $[-1.49, 1.08]$ & $x, y$ rescaled ($0.1$)\\
      & &$v_y$ &$\sim \mathcal{U}(-5, 5)$& $[-1.43, 1.39]$ & \\
      & &$C_1$ & - & $[1.0, 3.85]$ \\ \hline
      \multirow{4}{*}{Double Pendulum} & &$\theta_1$ & - & $[-1.39, 1.42]$ & \\
      & $C_{\text{ideal}} = L_1^2(m_1+m_2)\omega^2 + m_2L_2^2\omega^2$ &$\theta_2$ & - & $[-2.14, 2.15]$ & \multirow{2}{*}{$\omega_1, \omega_2$ rescaled ($0.1$)}\\
      & $ + 2m_1m_2L_1L_2\omega_1\omega_2\cos(\theta_1-\theta_2) $ &$\omega_1$ & - & $[-10.60, 10.46]$ & \\
      & $ - 2gL_1(m_1+m_2)\cos(\theta_1)-2gm_2L_2\cos(\theta_2)$ &$\omega_2$ & - & $[-21.21, 21.37]$ & 
      \label{table:1}
    \end{tabular}
  \end{ruledtabular}
\end{table*} 

We prepare total $6$ systems for training; $S1$, $S2$, $S3$, Lotka-Volterra system, Kepler problem, and observation data from the real double pendulum. Table \ref{table:1} shows the exact range and sample distribution of each variable in every model system.
We find that normalizing the data to match the scale between variables improves overall performances, and thus variables with maximum values exceeding $10$ are rescaled by a factor of $0.1$. Rescaling inputs also encourages the model output to be more linear with the true invariant; although unsupervised neural models can learn an arbitrary function of invariant $h(C(x))$, the output can still be linearized as $h(C(x)) = aC(x)+b$ with constants $a$ and $b$ if the output range is restricted to a small region.

\subsection{Synthetic system}
Dataset is composed by first randomly draw the relative variables except for the final one, and calculate the last relative variable which preserves the overall conserved quantity. We tried to maximize the variety of simulated data by setting noise distribution and variable range for each system as different as possible. While producing the dataset, we restricted the absolute value of the output of the final variable which consequently rejected some perilous set of variables that forces the last variable into an extremely diverging value.

\subsection{Physical system}

We generate the data from physical systems by integrating respective differential equations with Euler's method and performed subsampling to the trajectory. For the Lotka-Volterra system, we simulate the dynamics for $100M$ steps for the dataset of batch size $M$ with time interval $dt=0.01$. The obtained data are further subsampled for every $100$ steps, effectively setting the time interval between data points to $1$. We scale $x, y$ in the Lotka-Volterra equation and position coordinate $x, y$ in Kepler's system by a factor of $0.1$.

\subsection{Real double pendulum}
Double pendulum data is adopted from \cite{schmidt2009distilling}, where two trials for double pendulum data are available from the provided dataset in the Supplementary dataset. We use first trial, consists of $818$ data points with four-dimensional time-series $(\theta_1, \theta_2, \omega_1, \omega_2)$. Each data point corresponds to $0.01s$, making the total data length to $8.18$ seconds. The training set consist of the first $654 = 818 \times 0.8$ points and the test set consist of the rest $154 = 818 \times 0.2$ points. We scale $\omega_1, \omega_2$ by a factor of $0.1$ to match the scale with $\theta_1$ and $\theta_2$.

\section{Training Detail}
Both SNN and ConservNet are composed of $6$ layers of multi-layer perceptrons (MLP) with a layer width of $320$, where its input dimension varies by the target system and has a single output neuron. Note that the original SNN \cite{iten2020discovering} used $2$ layers of MLP and a layer width of $160$, and we found that increasing layer depth and width generally increases overall performances for both SNN and ConservNet. 

While training, we found the training result of SNN significantly varies by random initialization and highly prone to overfit. We report that SNN showing a good performance (train, test accuracy of $100\%, 99.75\%$ with correlation $0.998$) in one trial while converges to a meaningless output and strongly overfits (train, test accuracy of $100\%, 50.03\%$ with correlation $0.154$) in the very next trial with the same conditions. In several trials, larger layer width ($320$) exhibits strong overfitting while a relatively shallow one ($160$) shows better generalization performances.

For a fair comparison, we tested every combination of learning rates $[0.005, 0.0005, 0.00005]$ and layer width $[160, 320]$, and reported the best performing one (in terms of test accuracy) among $5$ trials for each condition as a representative result of SNN. In the case of ConservNet, we fixed the learning rate as $0.00005$ and layer width as $320$ neurons since the performance was robust to both layer width and learning rate. We trained both SNN and ConservNet for $50,000$ epochs with early stopping, Adam\cite{kingma2014adam} optimizer, and no particular regularizer. Batch size for mini-batch training is fixed to $64$ for SNN and tentative for Conservnet, where its batch size is fixed as respective group size $M$. Training takes several hours on a single GeForce GTX 1080, depends on the batch size.

\section{Extraction of symbolic formula from ConservNet results}

In the real scenario, identifying the explicit functional form of the invariant rather than just a numerical value of the invariant is often crucial for understanding the system and its inherent symmetry. This can be done by performing polynomial or symbolic regression to the ConserveNet output as mentioned in the main manuscript, but many of these regression methods are prone to overfitting if the data is errorneous. Hence, beside its usefulness, successfull retreival of symbolic function from data also indicates the high quality of the model output.
 As an exemplary case, we perform ridge regression with polynomial features of the input data on the output of ConservNet for invariant $S1$. The result of the regression for order $2$ is 

\begin{align}
  F_\theta &= 0.3069x_1 + 0.0005x_2 - 0.0019x_3 + 0.0035x_4 \\  \nonumber &+ 0.0005x_1^2 + 0.0006x_1x_2 - 0.0006x_1x_3 + 0.0002x_1x_4 \\ \nonumber &- 0.0006x_2^2 - 0.921x_2x_3 - 0.0002x_2x_4 + 0.0008x_3^2 \\  \nonumber &+ 0.0004x_3x_4 + 0.1528x_4^2 \\  \nonumber &\approx 0.3x_1 - 0.9x_2x_3 + 0.15x_4^2   \label{eq:S1}
\end{align}

We can see that the approximated output is the same as $\frac{3}{10}S1$, showing that ConservNet can provide reliable output for the extraction of symbolic formula.

\section{ConservNet on a system with multiple invariants}

\begin{figure}
  \centering
  \includegraphics[width=0.6\linewidth]{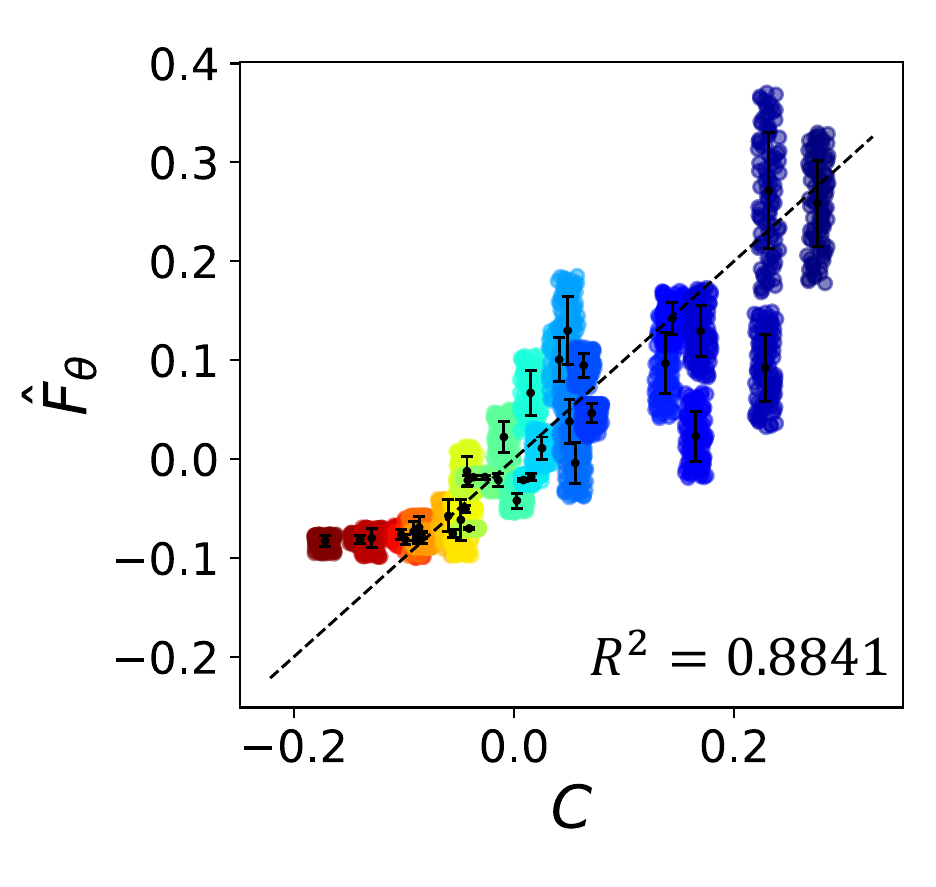}
  \caption{Ground-truth invariants $C$, the total energy, versus fitted ConservNet outputs $\hat{F}_\theta = aF_\theta+b$ for Kepler problem is plotted under data condition $(20, 100)$ with $R^2$ value, with controlled dataset where angular momentum over the dataset is fixed to $L=-1.5$. Here, $(a,b)=(-2.61, -0.29)$.}\label{fig:6}
\end{figure}

In the main manuscript, we show that ConservNet discovers the angular momentum from the Kepler problem among possible three invariants. Since the model is designed to output a single value from a single output, jointly finding multiple invariants needs a modification to the current architecture. In \cite{wetzel2020discovering}, the authors verified their model (SNN) by fixing the angular momentum of the given dataset and perform the same training. We test our model with similar settings with the controlled dataset to examine whether the model can discover the second invariant. Fig. \ref{fig:6} shows that the model output shows a strong correlation with the second invariant, the total energy of the orbital system from the controlled dataset. One possible reason for slightly worse performance compared to the case of angular momentum, which coincides with the result of \cite{wetzel2020discovering}, might be a massive scale difference between angular momentum and total energy in the dataset. For a stable periodic orbit in our simulation for the Kepler problem, we find that angular ranges from $-2$ to $2$ while the total energy ranges from $-0.2$ to $0.2$, approximately $10$ times smaller than the angular momentum. Since NV loss effectively increases gradient of the model output, function with larger gradient value might be more preferable to the model. Finding a particular modification of ConservNet for simultaneous discoveries of multiple invariants would be an interesting future direction.

\section{Robustness of ConservNet performances in various conditions}

\subsection{Results on various data conditions}

\begin{figure}
  \centering
  \includegraphics[width=0.6\linewidth]{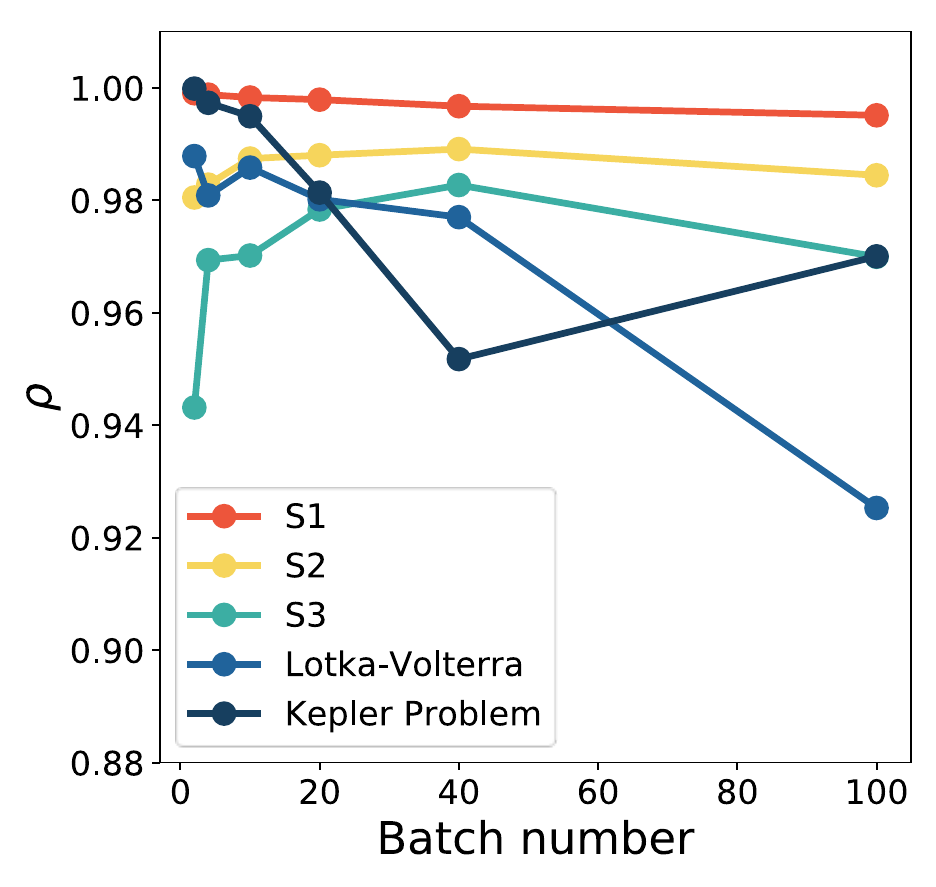}
  \caption{Pearson corrleation $\rho$ with varying data condition while training $50,000$ spochs of ConservNet for the invariants. Here, the total data number is fixed to $2,000$ and hence the data conditions of $(2, 1000), (4, 500), (10, 200), (20, 100), (40, 50)$ and $(100, 20)$ are tested.}\label{fig:7}
\end{figure}

In the main manuscript, we fixed all of the simulated data condition to $(20,100)$. In this section, we further test different data conditions with different batch numbers and batch sizes while fixing the total number of data points to $2,000$. In Fig. \ref{fig:7}, the correlation $\rho$ for all five simulated systems with different data conditions are plotted. We can confirm that ConservNet shows good performances ($\rho>0.9$) for all simulated settings, include both extreme ends; from the points where only a single, long dataset is possible to the points where a hundred of different trials with a short period of observation was recorded.

\subsection{Hyperparameters and spreader selection}

\begin{figure}
  \centering
  \includegraphics[width=0.6\linewidth]{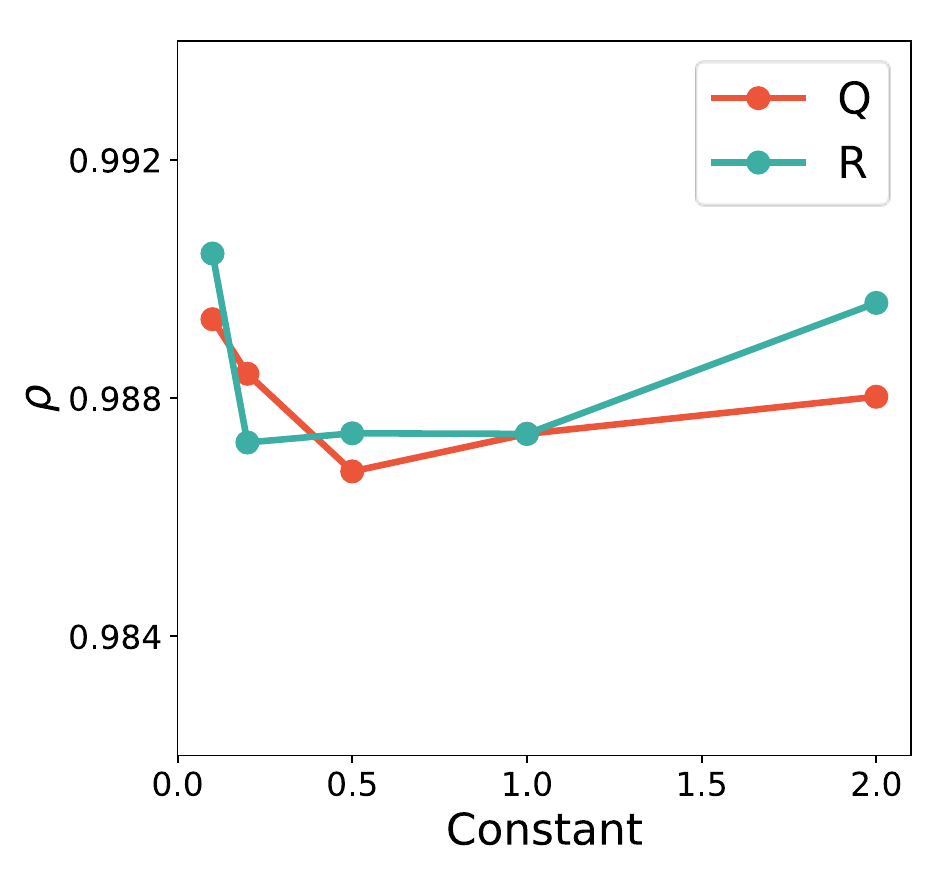}
  \caption{Pearson corrleation $\rho$ with varying spreading constant $Q$ (fixed $R=1$) and max noise norm $R$ (fixed $Q=1$) while training $20,000$ spochs of ConservNet for the invariant $S2$.}\label{fig:8}
\end{figure}

\begin{figure}
  \centering
  \includegraphics[width=0.6\linewidth]{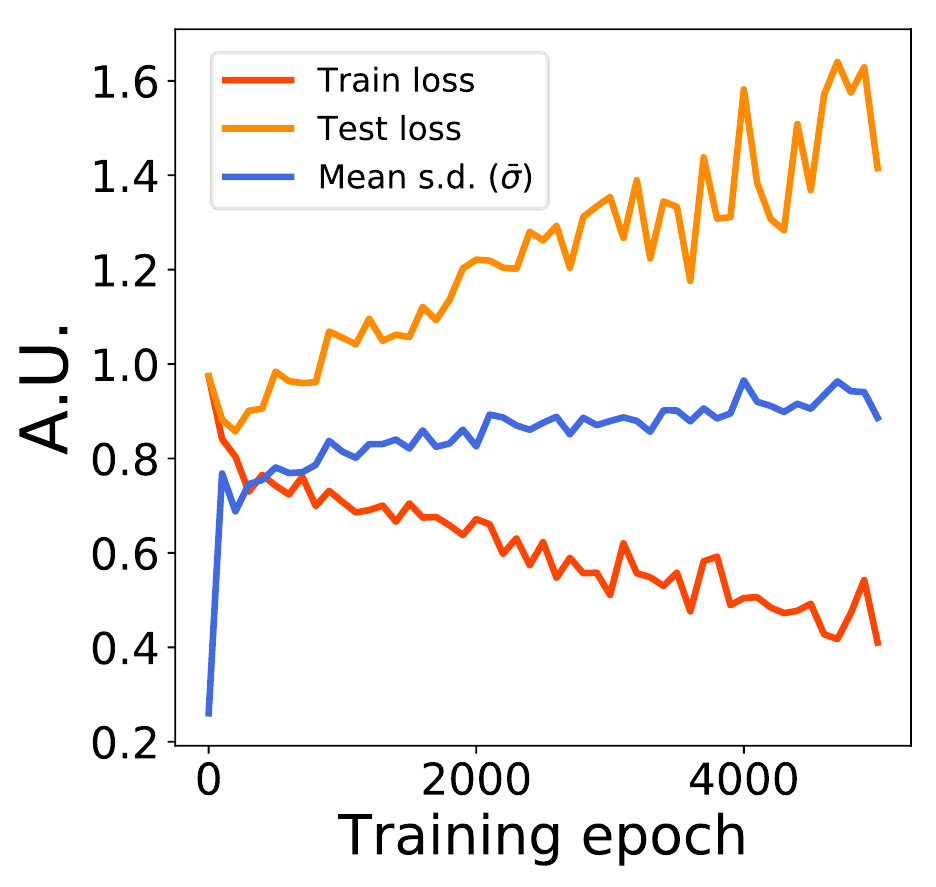}
  \caption{Train loss, test loss, and mean intra-group standard deviation ($\bar{\sigma}$) while training $5,000$ epochs of ConservNet for the random data.}\label{fig:9}
\end{figure}

In the main manuscript, we use $Q=1$ and $R=1$ for all experiments. But in our NV loss, the relative magnitude between $Q$ and $R$ determines the strength of the spreader, and one may raise a question about the relationship between such specific choice of hyperparameters and the model performance. To test the robustness of ConservNet with hyperparameters, we use $S2(20,100)$ as a test invariant and record the Pearson correlation $\rho$ while varying $Q$ and $R$ values. The results are shown in Fig. \ref{fig:8}, which verifies that the model performances are practically unaffected by the choice of specific hyperparameters.

We also test different types of spreaders by restraining the noise with different norms. Instead of $L_2$-norm we use in the main experiment, we test $L_1$-norm and $L_{\infty}$-norm to be restrained to $1$. As a result, both spreaders achieve comparable results of $0.9910$ and $0.9862$ for $S2(20,100)$. To sum up, ConservNet is robust with both hyperparameters selection and spreader type, implying that the spreader can be freely constructed as long as it serves the main purpose; preventing the model from trivial convergence.

\subsection{System with no specific invariant}

If the system has no specific invariant, a good model for invariant discovery should notice such absence. We train ConservNet for the random data consists of $5$-dimensional Gaussian random vector $X \sim \mathcal{N}(\mathbf{\mu}, \mathbf{\Sigma})$, where $\mathbf{\mu} \in \mathbb{R}^5$ and $\Sigma = I_5$. ConservNet alarms the absence of invariant by showing strong overfitting and large intra-group deviations, as shown in Fig. \ref{fig:9}.

\bibliography{ConservNet}

\end{document}